\documentclass[letterpaper]{article} 
\usepackage{aaai25}  
\usepackage{times}  
\usepackage{helvet}  
\usepackage{courier}  
\usepackage[hyphens]{url}  
\usepackage{graphicx} 
\usepackage{natbib}  
\usepackage{caption} 
\usepackage{algorithm}
\usepackage{algorithmic}

\usepackage{xcolor}         
\usepackage{booktabs}
\usepackage{amsmath}
\usepackage{amssymb}
\usepackage{subfloat}
\usepackage{subfig}
\usepackage{graphicx} 
\usepackage{multirow}

\usepackage{amsthm}
\newtheorem{definition}{Definition}
\newtheorem{theorem}{Theorem}

\newcommand{\distas}[1]{\mathbin{\overset{#1}{\kern\z@\sim}}}%

\usepackage{tikz}
\newcommand*\circled[1]{\tikz[baseline=(char.base)]{
            \node[shape=circle,fill,inner sep=1pt] (char) {\footnotesize \textcolor{white}{#1}};}}

\usepackage{newfloat}
\usepackage{listings}
\DeclareCaptionStyle{ruled}{labelfont=normalfont,labelsep=colon,strut=off} 
\floatstyle{ruled}
\newfloat{listing}{tb}{lst}{}
\floatname{listing}{Listing}
\title{Unifying Explainable Anomaly Detection and Root Cause Analysis in Dynamical Systems}
\author{
    Yue Sun,     
    Rick S. Blum,
    Parv Venkitasubramaniam\\
}
\affiliations{
    Lehigh University \\

        \{yus516, rb0f, pav309\}@lehigh.edu
}

\usepackage{bibentry}

\begin{document}

\maketitle

%

\begin{abstract}
Dynamical systems, prevalent in various scientific and engineering domains, are susceptible to anomalies that can significantly impact their performance and reliability. This paper addresses the critical challenges of anomaly detection, root cause localization, and anomaly type classification in dynamical systems governed by ordinary differential equations (ODEs). We define two categories of anomalies: cyber anomalies, which propagate through interconnected variables, and measurement anomalies, which remain localized to individual variables. 
To address these challenges, we propose the Interpretable Causality Ordinary Differential Equation (ICODE) Networks, a model-intrinsic explainable learning framework. ICODE leverages Neural ODEs for anomaly detection while employing causality inference through an explanation channel to perform root cause analysis (RCA), elucidating why specific time periods are flagged as anomalous.
ICODE is designed to simultaneously perform anomaly detection, RCA, and anomaly type classification within a single, interpretable framework. Our approach is grounded in the hypothesis that anomalies alter the underlying ODEs of the system, manifesting as changes in causal relationships between variables. We provide a theoretical analysis of how perturbations in learned model parameters can be utilized to identify anomalies and their root causes in time series data. Comprehensive experimental evaluations demonstrate the efficacy of ICODE across various dynamical systems, showcasing its ability to accurately detect anomalies, classify their types, and pinpoint their origins. 

\end{abstract}

\section{Introduction}

\begin{figure}
\centering
\includegraphics[width=0.99\columnwidth]{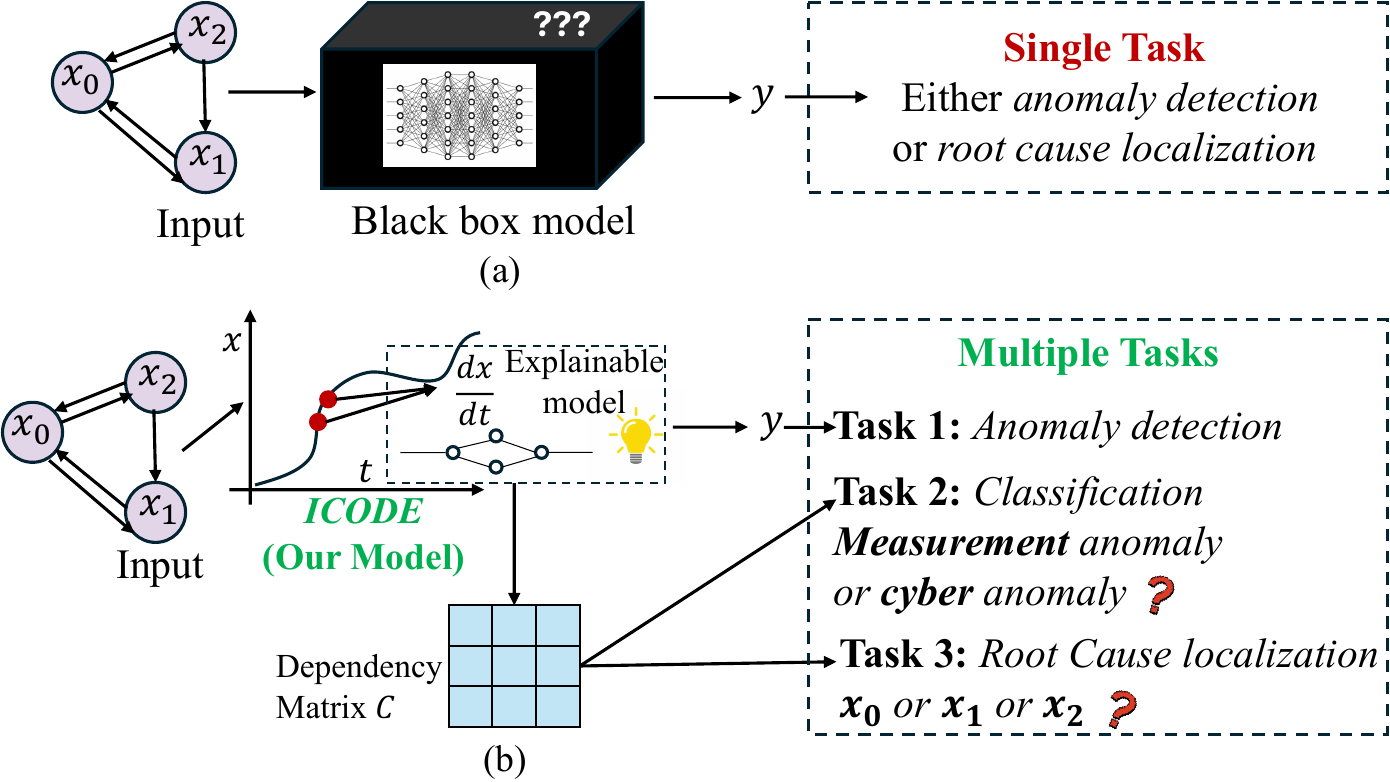} 
\caption{(a) Traditional models typically perform either anomaly detection or root cause localization in isolation, limiting their comprehensive analysis capabilities.
(b) Our proposed model integrates anomaly detection, anomaly type classification and root cause localization within a single model. It achieves this by leveraging both the model output and the learned dependency matrix, which encapsulates the causal relationships between variables in the system.
}.

\label{fig:introduction}
\vspace{-0.27in}
\end{figure}

Dynamical systems, which describe the evolution of variables over time, are fundamental to numerous domains, including natural sciences such as physics~\cite{willard2020integrating}, biology~\cite{laurie2023explainable}, and chemistry~\cite{keith2021combining}, as well as engineering fields like transportation systems~\cite{shepherd2014review} and power systems~\cite{ju2021fault}. 
These systems often encounter various types of anomalies, each with potentially different levels of urgency and impact~\cite{gao2022detection}. 
For instance, in a traffic management system, anomalies can range from minor sensor malfunctions to major incidents like accidents or road closures. 
The urgency and impact of these anomalies can vary significantly: a critical accident requires immediate attention and may have widespread effects on traffic flow across the network, whereas a single malfunctioning speed sensor might not necessitate urgent intervention if the overall transportation network is functioning normally. 
Similarly, in a power grid, a sudden voltage spike could indicate a potentially catastrophic equipment failure requiring immediate action, while a gradual drift in a non-critical sensor reading might be addressed during routine maintenance. 
This variability in urgency and impact underscores the importance of not only detecting anomalies but also comprehensively understanding their origins and potential consequences, a process commonly referred to as root cause analysis (RCA).

Machine learning (ML) based anomaly detection~\cite{pang2021deep} and RCA algorithms~\cite{soldani2022anomaly} have emerged as state-of-the-art approaches due to their capability to learn complex functions through data-driven methods~\cite{teng2010anomaly, blazquez2021review} . 
However, existing models face several challenges. 
First, deep neural network-based models~\cite{choi2021deep, li2023deep} typically require large datasets to train complex ML networks, often producing only binary results indicating the presence or absence of an anomaly. 
This is insufficient for many applications where pinpointing the underlying causes of system faults is equally important ~\cite{wang2023incremental}. 
Second, beyond localizing the root cause, understanding the type of anomaly is crucial for determining its urgency and appropriate response. 
Existing models, particularly complex deep learning based models, cannot be used to differentiate types of anomalies as relevant to dynamical systems in consideration, thus limiting their scope of application.
Last, many dynamical systems adhere to differential equations, complicating the learning process and making it difficult to develop models that can accurately detect anomalies and perform RCA in such systems.

In this work, we address these challenges by first mathematically defining two types of anomalies that are commonly observed in dynamical systems~\cite{luo2021deep}: \textit{cyber} anomalies and \textit{measurement} anomalies. 
While these types of anomalies are known to exist in practice, our contribution lies in providing a mathematical framework for their definition and analysis. 
Measurement anomalies are typically caused by errors during data acquisition and are confined to local faults. 
For example, in a power grid, if a voltage sensor at a single substation displays an incorrect reading due to calibration issues, it would be considered a measurement anomaly. 
Conversely, cyber anomalies result from interactions between variables, with the root cause potentially affecting neighboring sensors. 
In the power grid context, a cyber anomaly could manifest as a cascading voltage instability, where a fault in one part of the network propagates and affects voltage readings across multiple interconnected substations. 
This mathematical distinction is crucial for understanding the scope and potential impact of detected anomalies and forms the foundation for our subsequent analysis and model development.

Our primary objective is to develop an interpretable ML model capable of performing anomaly detection, type classification and RCA, within a single model, based on causality inference, as shown in Figure~\ref{fig:introduction}. 
Our driving hypothesis is that if a variable in a dynamical system adheres to an ordinary differential equation (ODE), an anomaly would alter the ODE, which can be monitored by changes in the causality relationship (i.e., changes in the causal graph weights learned from data). 
Simultaneously, changes in this relationship can indicate the type of error. Consequently, we reframe the problem as causality inference with high accuracy within the corresponding ODE system.

Inspired by Neural ODEs~\cite{chen2018neural}, we introduce the Interpretable Causality Ordinary Differential Equation (ICODE) Networks, a novel prediction-based model for anomaly detection and root cause analysis in dynamical systems. 
ICODE utilizes Neural ODEs to capture system evolution, learning the governing ODEs directly rather than the integration process, which enhances both accuracy and efficiency in modeling dynamical systems. 
The model detects anomalies by identifying significant deviations between predicted variables and ground truth, indicating substantial changes in the underlying ODEs. 
Crucially, ICODE's interpretability stems from its ability to analyze the learned parameters, which represent causality relationships within the system. 
Upon detecting an anomaly, this analysis enables the determination of both the type and location of the root cause, providing valuable insights into system behavior during anomalies.
\begin{figure}
\centering
\includegraphics[width=0.99\columnwidth]{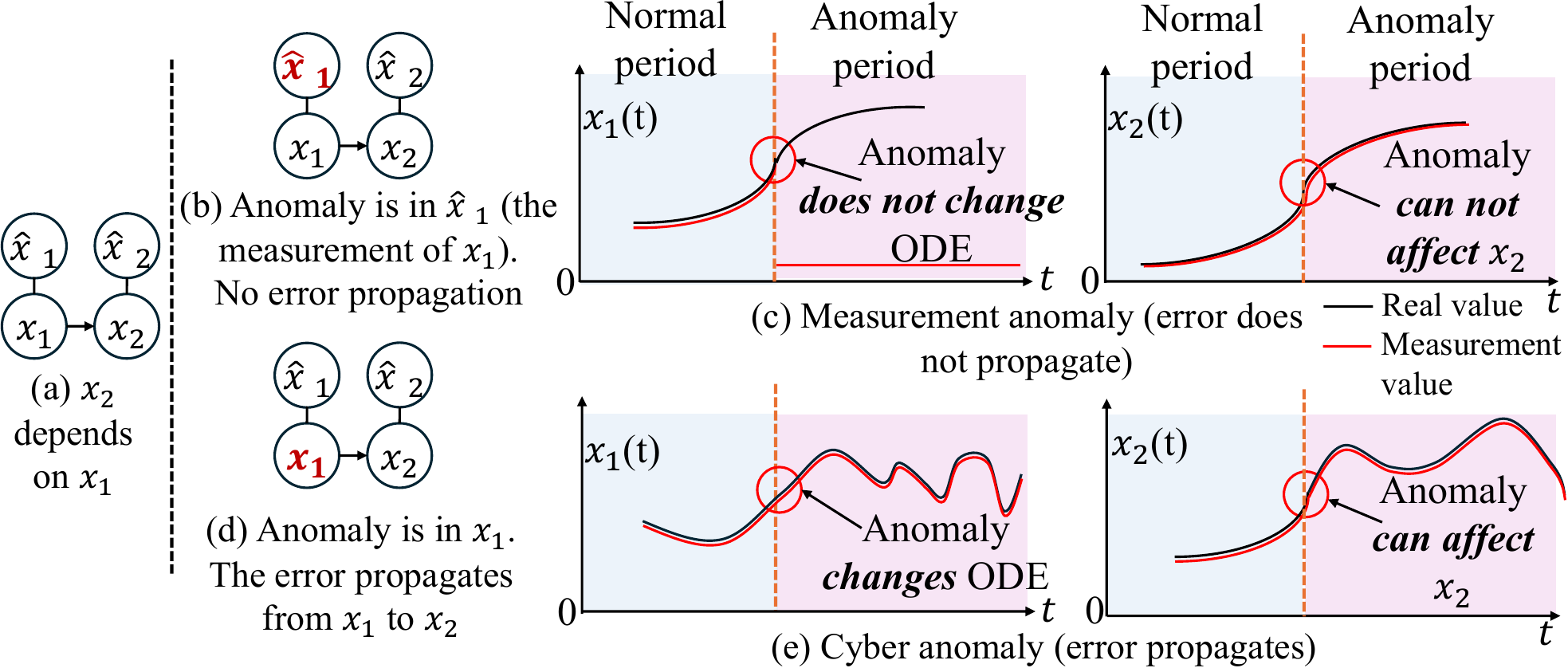} 
\caption{Comparison of Measurement and Cyber Anomalies in a Two-variable System
(a) Normal state.
(b) Measurement anomaly affects the measurement of a variable without altering underlying ODEs.
(c) Measurment anomaly does not propagate to dependent variables.
(d) Cyber anomaly alters the underlying ODE and lead to an anomaly. 
(e) Cyber anomaly propagates anomalies to dependent variables.
}.
\label{fig:problem-def}
\vspace{-0.27in}
\end{figure}

We summarize our contributions as follows:
\begin{itemize}
    \item We define cyber anomalies and measurement anomalies based on whether the root cause affects dependent variables or remains localized.
    \item We propose an interpretable ML model called Interpretable Causality Ordinary Differential Equation (ICODE) Networks, which utilizes Neural Ordinary Differential Equations Networks for anomaly detection, anomaly type classification and root cause localization.
    \item We provide a theoretical analysis of the pattern of changes in learned parameters that can be used to localize root causes and classify the anomaly type in dynamical systems data when the ODE is altered.
    \item Extensive experiments on three simulated ODE systems demonstrate ICODE's superior performance in anomaly detection, root cause localization, and anomaly type classification compared to state-of-the-art methods.
\end{itemize}

\section{Related Work}

\textbf{Anomaly detection models.}
Deep learning models~\cite{pang2021deep} have been widely adopted for anomaly detection due to their powerful capabilities. 
Variational autoencoder (VAE) based models~\cite{zhou2021vae} are utilized to learn latent representations for this task. 
The Anomaly Transformer~\cite{xu2021anomaly} demonstrates significant effectiveness by jointly modeling pointwise representations and pairwise associations. 
Additionally, a diffusion model~\cite{xiao2023imputation} has been proposed to detect anomalies through data imputation. 
However, as these models grow in size and complexity, they become increasingly opaque, providing only binary detection results without insight into the underlying decision-making process. 
Moreover, these models typically require large datasets, which may not be feasible in many real-world applications.

\begin{figure*}
\centering
\includegraphics[width=0.81\textwidth]{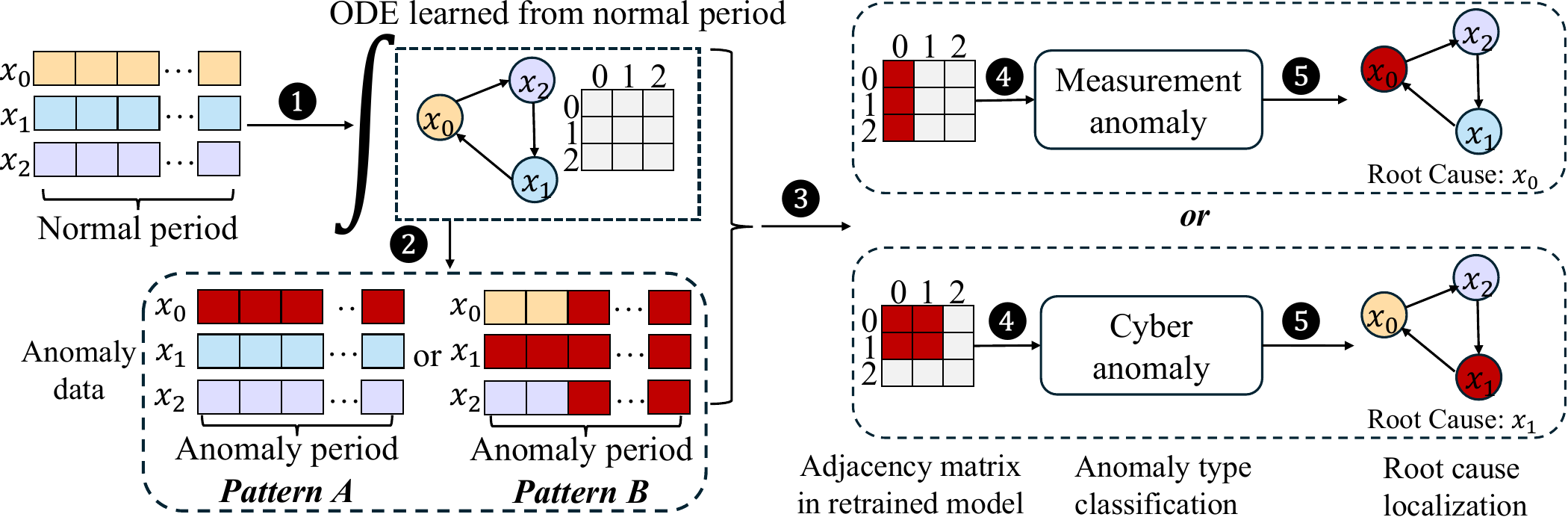} 
\caption{A 3-variable example dynamical system. 
$\protect\circled{1}$ Train ICODE using data from normal periods to establish the baseline causality relationship $C$.
$\protect\circled{2}$ Apply the trained ICODE to detect anomalies in anomaly period.
$\protect\circled{3}$ Retrain the model using the anomalous data to obtain a new causality relationship $C'$.
$\protect\circled{4}$ Analyze the causality difference $|C-C'|$ to classify the anomaly type; $\protect\circled{5}$ Apply Eq. (\ref{eq:get-root-cause}) or Eq. (\ref{eq:get-root-cause-2}) to identify the root cause.
}.

\label{fig:architecture}
\vspace{-0.27in}
\end{figure*}

\textbf{Root Cause Analysis models.}
Root cause analysis (RCA) using machine learning~\cite{lei2020applications, richens2020improving} has gained increasing popularity in recent years. 
For example, a variational autoencoder model~\cite{han2023root} leverages a known causal graph to identify root causes based on Pearl's structural causal model. 
However, these models often require precise causal relationships and separate models for anomaly detection, which is challenging in practice. 
Obtaining accurate causal graphs is sometimes challenging, and deploying two individual models can be cost-prohibitive. 
Other RCA methods based on causal structures have been proposed~\cite{budhathoki2022causal, wang2023interdependent}, but they typically assume anomalies are outliers, 
which is not always valid, leading to potential inaccuracies in the learned causality architecture.

\textbf{Model Explainability using Neural ODEs.} 
Model-intrinsic-explainable models are designed with inherent transparency, allowing direct interpretation of their decisions and processes from their internal structures~\cite{Lakkaraju2016,Lou2012}. Self-explainable Neural Networks (SENN) employ explainable basis concepts for explicit and faithful predictions and classification~\cite{alvarez2018towards}. 
Building on SENN, the GVAR model learns causal relationships within time series data using an interpretable framework~\cite{marcinkevivcs2021interpretable}. 
Additionally, time series systems often follow differential equations, and Neural ODEs have been proposed to approximate these systems~\cite { jia2019neural, asikis2022neural}. 
However, using model-intrinsic-explainable models enhanced by Neural ODEs to learn patterns in dynamic systems for anomaly detection and RCA remains largely unexplored.

\section{Notations and Problem definition}
Consider a time series in a dynamical system with $p$ variables: $X(t) = \{ x_1(t), x_2(t), ..., x_p(t) \}$, which follows causal relationships represented by a directed graph $\mathcal{G} = \{ X(t), \mathcal{E} \}$. In this graph, each node $x_i$ corresponds to a variable in $X(t)$, and each edge $(i, j) \in \mathcal{E}$ indicates whether variable $x_i$ can influence variable $x_j$. 
We assume the evolution of variable $x_i$ satisfies a differential equation:
\begin{equation}
\label{eq:system}
\begin{aligned}
    \frac{dx_i(t)}{dt} & = f_i \left( X(t), \mathcal{G}  \right), \forall 1 \leq i \leq p,    \\
        x_i(t) & = x_i(0) + \int_0^t f_i \left( X(t), \mathcal{G}  \right) dt, 
\end{aligned} 
\end{equation}
where
where $f_i$ is an unknown nonlinear function specifying the evolution of $x_i$ at time $t$.
Sensors are deployed to measure variables, denoted as $\{ \hat{x}_1, \hat{x}_2, ..., \hat{x}_p\}$, with noise $\epsilon$:
\begin{equation}
    \hat{x}_i = x_i + \epsilon.
\end{equation}

We define two types of anomalies: the \textit{measurement} anomaly and the \textit{cyber} anomaly.
The measurement anomaly only affects the measurement of the variable but does not affect the differential equation, as illustrated in Figure~\ref{fig:problem-def}.

\begin{definition}
\textbf{Measurement Anomaly}: 
We define a system to exhibit a measurement anomaly in variable $i$ at time $t$, if the variable measurement at time $t$ is given by:
\begin{equation}
\label{eq:physical}
\begin{aligned}
    \hat{x}_i'(t) & = A(\hat{x}_i(t)),
\end{aligned}  
\end{equation}
where $A(\cdot)$ is an anomalous function such that $A(x) \neq x$.
\end{definition}
For our theoretical analysis, we assume that for a period of time, $A(\cdot)$ does not change.
Unlike measurement anomalies, cyber anomalies directly alter variables within the ODE, modifying the system's evolution and potentially propagating to dependent variables. 
Such anomalies can be caused by cyber-attacks, including false data injection~\cite{sayghe2020survey}, which fundamentally disrupt the system's dynamics.

\begin{definition}
\textbf{Cyber Anomaly}: 
We define a system to exhibit a cyber anomaly in variable $i$ at time $t$, with anomalous function denoted as $A(\cdot)$,
if variable $x_i$ is changed to 
\begin{equation}
    x_i'(t) = A(x_i(t)), A(x) \neq x.
\end{equation} 

\end{definition}

Given a $t$-length measurement of time series, we aim to propose a model-intrinsic-explainable model to conduct:
\begin{enumerate}
    \item Detect whether the system has an anomaly.
    \item Determine the type of root cause (measurement anomaly or cyber anomaly).
    \item Identify which variable is the root cause of the anomaly.
\end{enumerate}

We achieve this by proposing a model that can demonstrate
the change in the causality relationship given the $t$-length
time sequence.

\section{Methodology}
We propose a model-intrinsic-explainable framework for time series prediction called Interpretable Causality Ordinary Differential Equation (ICODE) networks. 
This model is designed to make predictions by learning the driven ODE in the dynamical system and inferring causality relationships within a time series through its explanation channel.
The ICODE model is defined as follows
\begin{equation}
\label{eq:ICODE}
\begin{aligned} 
\frac{dX(t)}{dt} & = \Phi_{\theta}(X(t))X(t) + b, \\
\hat{X}(t+1) & = X(t) + \int_{t}^{t+1} \left( \Phi_{\theta}(X(\tau))X(\tau) + b\right) d\tau,    
\end{aligned}
\end{equation}
where $b \in \mathbb{R}^{p \times 1}$ is the bias term. 
$\Phi_{\theta}(X(t)):\mathbb{R}^p \rightarrow \mathbb{R}^{p \times p}$ is a neural network parameterised by $\theta$, learning the ODE governing the dynamical system.
The output of $\Phi_{\theta}(X(t))$ is an $p$ by $p$ matrix representing the causality dependency between variables, corresponding to the weight in causality graph $\mathcal{G}$.
Specifically, the element at $i, j$ of $\Phi_{\theta}(X(t))$ represents the influence of $x_j(t+\delta t)$ on $x_i(t)$, where $\delta t$ is a very small time interval.
A larger element at $i, j$ of $\Phi_{\theta}(X(t))$ indicates a high likelihood $(i,j)$ is an edge in $\mathcal{E}$.

As shown in Eq. (\ref{eq:ICODE}), the model could predict $X(t+1)$ based on the input $X(t)$ at time $t$, by letting the neural network $\Phi_{\theta}(\cdot)$ learn the evolution of $X(t)$.
For this reason, we employ Neural ODE, a machine learning framework that approximates the integral process, given the initial value, differential equation, and time interval. 
The prediction is then given by:
\begin{equation}
\begin{aligned}
\hat{X}(t+1)  = & NeuralODE(X(t), \\
& ( \Phi_{\theta}(X(t))X(t)
+ b), [t, t+1]).    
\end{aligned}
\end{equation}

\noindent The dependency matrix, representing relationships between variables in $X(t)$, can be explored through $\Phi_{\theta}(X(t))$.
Enforcing sparsity in $\Phi_{\theta}(X(t))$ is crucial, as it ensures that the learned model captures the immediate dependence between variables as expressed through the differential equation. 
This sparsity is essential for accurately representing the true causal structure. 
To achieve this sparsity in $\Phi_{\theta}(X(t))$, we train ICODE using the following loss function:
\begin{equation}
\label{eq:loss}
\frac{1}{T-1} \sum_{t=0}^{T-1} ( \hat{X}(t+1) - X(t+1) )^2 + \frac{\lambda}{T-1} \sum_{t=0}^{T-1} R(\Phi_{\theta}(X(t+1)),
\end{equation}
where the first term represents the mean squared error (MSE) of the one-step prediction using ICODE, and
$R(\cdot)$ is a sparsity-inducing penalty term.
$\lambda$ is a hyperparameter that balances the importance of sparsity.
The architecture of ICODE is shown in Figure~\ref{fig:architecture}.

\subsection{Anomaly Detection}
Anomalies in the system, whether cyber anomalies or measurement anomalies, cause significant deviations in the prediction $\hat{X}(t+1)$ from the ground truth $X(t+1)$. The anomaly score is calculated as follows:
\begin{equation}
\label{eq:anomaly-score}
AnomalyScore = \sum_{t=1}^{T} |\hat{X}(t) - X(t)|.
\end{equation}

\subsection{Anomaly Type Classification}
Our hypothesis is that the root cause of an anomaly is located at the variable where the causality relationship changes significantly. 
The causality relationship is inferred from $\Phi_{\theta}(\cdot)$~\cite{marcinkevivcs2021interpretable}:
\begin{equation}
\small
\label{eq:causility}
C_{i,j} = median_{0 \leq t \leq T-1} \left\{ | \Phi_{\theta, (i, j)}(X(t)) | \right\}, \textit{for } 1 \leq i, j \leq p.
\end{equation}
\normalsize
We compute the causality relationship matrix $C_{i,j}$ during normal periods and $C'_{i,j}$ during anomaly periods,
where a higher value in $C_{i,j}$ and $C'_{i,j}$ represents a variable of high likelihood $i$ that influences variable $j$ during normal and anomaly period respectively.

The inherent differences between cyber anomalies and measurement anomalies manifest in how they affect the underlying ODEs and, consequently, the causality relationships learned by ICODE. This distinction forms the basis for our anomaly type classification.
To determine the type of anomaly, we calculate a measurement anomaly score $M$, which indicates the probability that observed causality changes are due to a measurement anomaly. The score is given by:
\begin{equation}
\label{eq:anomaly-type}
M(C, C') = \max_i \sum_{j=1}^p \frac{\mathbb{I}\left(|C_{i,j}-C'_{i, j}| \geq \Bar{\gamma}\right)}{p},
\end{equation}
where $C$ and $C'$ are the causality matrices in non-anomalous and anomalous data, $\Bar{\gamma}$ is a threshold used to select the $m$ largest values, $\mathbb{I}$ is an indicator function used to select the top $m$ elements in $|C-C'|$, and $p$ is the number of variables.

The calculation process involves calculating the proportion of the largest $m$ elements in the difference matrix $|C-C'|$ for each row, and taking the maximum proportion as the measurement score $M(C, C')$. A greater $M(C, C')$ indicates a higher likelihood that the anomaly is a measurement anomaly, since the changes are concentrated on one column. While a lower score suggests a cyber anomaly. This classification is crucial for subsequent root cause localization, which we will discuss in the following section.



\subsection{Root Cause Localization}

The root cause score representing the likelihood of variable $i$ being the root cause of measurement anomalies is given by:
\begin{equation}
\label{eq:get-root-cause}
    S(i) = \left( 
    \sum_{j=1}^{p}|{C}_{i,j} - {C}'_{i,j}| + \sum_{j=1}^{p}|{C}_{j,i} - {C}'_{j,i}|
    \right)_i.
\end{equation}
The root cause location is indicated by the largest combined change across both the row (first term) and column (second term) sums.
Whereas the root cause score of variable $i$ causing cyber anomalies is given by:
\begin{equation}
\label{eq:get-root-cause-2}
    S'(i)=\sum_{k \in \{k:{C}_{i, k}=1 \ or \ {C}_{k, i}=1\}}S(k),
\end{equation}
representing the sum of the changed weight that is related to variable $i$ according to the learned causality graph.


\section{Theoretical Analysis}
In this section, we provide theoretical support for how measurement and cyber anomalies distinctly affect the weights of the causal graph, as represented by the properties of $|C - C'|$, when finite difference methods are used to approximate the integration in Eq. (\ref{eq:ICODE}).
The fundamental difference between these anomaly types lies in their position relative to the integral process in the dynamical system.
Measurement anomalies occur outside the integral function, leading to a localized impact on the causal graph. We formalize this observation in the following theorem:


\begin{theorem}
\label{theorem:measurement}
Consider a dynamical system $X(t)$ as learned using Eq. (\ref{eq:ICODE}), where $\mathcal{G}=(X(t), \mathcal{E})$ denotes the underlying dependency graph. 
Let $C(\cdot,\cdot)$ and $C'(\cdot,\cdot)$ denote the causality matrices extracted using non-anomalous and anomalous data, respectively. 
Assuming $C(i,j)\neq 0$ iff $(i,j)\in \mathcal{E}$, 
and {the diagonal elements of $C$ are much larger than the off-diagonal elements,} $C(i,i) \gg C(i,j), i \neq j$, then
(1) If the anomaly is a measurement anomaly at variable $k$, then $|C(k,j)-C’(k,j)| \gg |C(k',j)-C’(k',j)| > 0$, where $(k, k') \in \mathcal{E}$.
(2) If the anomaly is a cyber anomaly at variable $k$, then $|C(k,j)-C’(k,j)| > 0$ only if $(k,j)\in \mathcal{E}$.
\end{theorem}
\begin{proof}
    See Appendix.
\end{proof}

Theorem~\ref{theorem:measurement} indicates that a measurement anomaly with a root cause in variable $i$ only alters the values of $|C - C'|$ in column $i$ or row $i$ of the causality matrix. 
This localized effect is consistent with the nature of measurement anomalies, which implies that the dependency of the affected variable on any other variable can be altered regardless of ground truth dependence.
In contrast, cyber anomalies occur within the integral function, allowing their effects to propagate to dependent variables. 
Theorem~\ref{theorem:measurement} also suggests that a cyber anomaly with a root cause in variable $i$ can alter the values of $|C - C'|$ for all edges connected to variable $i$ and variable $j$, where the edge $(i,j)$ in causal graph $\mathcal{G}$. 
This more extensive impact reflects the propagating nature of cyber anomalies through interconnected variables in the system.

\section{Experiments}
\label{experiments}

The distinct patterns of change in the causality relationships $|C - C'|$ for measurement and cyber anomalies provide a robust foundation for anomaly type classification. Measurement anomalies produce a concentrated change in the causality matrix, primarily affecting a single variable's relationships, aligning with scenarios like sensor errors or calibration issues. Conversely, cyber anomalies result in a more diffuse pattern of changes, reflecting their ability to propagate through interconnected systems. These theoretical insights not only support our anomaly classification approach but also provide a deeper understanding of how different types of anomalies manifest in dynamical systems, enabling more targeted and appropriate responses to system irregularities.
The Purpose of our experiments is 
\begin{itemize}
    \item We evaluate ICODE's performance in identifying anomalies within dynamical systems, comparing it to state-of-the-art methods.
    \item We examine the model's capability to accurately pinpoint the location of the root cause, comparing it to state-of-the-art methods.
    \item  We demonstrate ICODE's unique ability to categorize anomalies.
\end{itemize}

\subsection{Experiment Settings}
Our experiments involve three simulated ODE systems, each comprising $20$ variables. For each system, we generate $100,000$ data points spanning normal, cyber anomaly, and measurement anomaly periods using specified ODEs. In both the cyber and measurement anomaly simulation process, we create $500$ data points for each anomaly instance, with every set of $500$ points sharing a common root cause. 
To form our final datasets, we downsample each period individually to $10,000$ points.

\subsubsection{N-species Lotka-volterra System}
The original Lotka–Volterra system~\cite{bacaer2011short} are used to model the relationship among $p$ predator and prey species, and the corresponding population sizes are denoted by $X(t) = \{x_1(t), x_2(t), ..., x_p(t) \}$. 
Population dynamics are given by the following coupled differential equations:
\begin{equation}
\begin{aligned}
    \frac{dx_i}{dt} = r_i x_i\left( 1 - \frac{\sum_{j=1}^{p} \beta_{ij}x_j}{K_i} \right),
\end{aligned} 
\end{equation}
where $r_i$ is inherent percapita growth rate, and $K_i$ is the carrying capacity, and $\beta_{ij}$ represents the effect of species $j$ on species $i$.
We initialize $\beta_{ij}$ randomly. 

\subsubsection{Lorenz-96 system}
Consider the Lorenz-96 system with $n$ variables denoted as $X(t) = \{x_1(t), x_2(t), ..., x_p(t) \}$, following an ordinary differential equation:
\small
\begin{equation}
\label{eq:lorenz96}
    \frac{dx_i(t)}{dt} = (x_{i+1}(t) - x_{i-2}(t))x_{i-1}(t) - x_i(t) + F, \forall 1 \leq i \leq n,
\end{equation}
\normalsize
where $x_{0} := x_p$, $x_{-1} := x_{p-1}$, and $F$ is a forcing constant. We set $F=10$.

\subsubsection{Reaction-diffusion system}

Reaction-diffusion systems consist of a local reaction and a diffusion, which cause the substances to spread out, and are used in modeling many domains in biology~\cite{fisher1937wave}, geology, and transportation. 
We pick up one reaction-diffusion system with $n$ variables denoted as $X(t) = \{x_1(t), x_2(t), ..., x_p(t) \}$, while the neighboring variables are connected by a diffusion process (e.g., $x_i$ is connected to $x_{i-1}$ and $x_{i+1}$):
\begin{equation}
    \frac{dx_i}{dt} = (x_{i-1}-x_{i}) + (x_{i+1}-x_{i}) + x_i(1-x_i),
\end{equation}
where the first two terms represent a diffusion process with diffusion coefficients $1$, and the third term represents the local reaction called Fisher's equation~\cite{fisher1937wave}.

\subsubsection{Anomaly Simulation}
We simulate the anomaly using the following equation: 
\begin{equation}
\label{eq:anomaly-value}
    A(x_i) = x_i + a, a \sim \mathcal{N}(\alpha, 1)
\end{equation}
where $a$ is a random variable from a Gaussian distribution with mean $\alpha$ and variance $1$.
We differentiate between measurement and cyber anomalies. 
Measurement anomalies are added to one variable after time series simulation, without altering the underlying ODE. 
Cyber anomalies, conversely, are incorporated during time series simulation, changing the underlying ODE.

\begin{table*}[!h]
\scriptsize
\centering
\caption{Performance of ICODE compared to baseline models in anomaly detection.}
\begin{tabular}{llcccccccccc}
    \midrule
    & & & \multicolumn{3}{c}{\textbf{Lotka-valterra}} & \multicolumn{3}{c}{\textbf{Lorenz-96}} & \multicolumn{3}{c}{\textbf{Reaction-diffusion}} \\
    \cmidrule(r){3-12}
    & & & Precision & Recall & F1 & Precision & Recall & F1 & Precision & Recall & F1  \\
    \midrule
    \multirow{3}{*}{$\alpha = 0.5$} & Deep SVDD & & 0.9033 & 0.9102 & 0.9067 & 0.9232 & 0.9019 & 0.9124 & 0.9235 & 0.9444 & 0.9338 \\
    & Anomaly Transformer & & \textbf{0.9284} & 0.9311 & \textbf{0.9307} & \textbf{0.9555} & 0.9434 & 0.9494 & 0.9514 & 0.9548 & 0.9514 \\
    & ICODE & & 0.9240 & \textbf{0.9363} & 0.9301 & {0.9475} & \textbf{0.9593} & \textbf{0.9533} & \textbf{0.9675} & \textbf{0.9624} & \textbf{0.9649} \\
    \midrule
    \multirow{3}{*}{$\alpha = 1$} & Deep SVDD & & 0.9224 & 0.9342 & 0.9282 & 0.9310 & 0.9326 & 0.9317 & 0.9468 & 0.9658 & 0.9562 \\
    & Anomaly Transformer & & \textbf{0.9374} & 0.9345 & 0.9359 & \textbf{0.9607} & 0.9568 & 0.9587 & 0.9815 & 0.9763 & 0.9788 \\
    & ICODE & & 0.9367 & \textbf{0.9472} & \textbf{0.9419} & {0.9595} & \textbf{0.9693} & \textbf{0.9594} & \textbf{0.9827} & \textbf{0.9772} & \textbf{0.9799} \\
    \midrule
\end{tabular}
\label{table:main-result}
\end{table*}

\subsubsection{Baselines}
For anomaly detection, we compare our model with Deep SVDD~\cite{zhou2021vae} and Anomaly Transformer~\cite{xu2021anomaly} to demonstrate comparable performance on simulated datasets. 
For root cause analysis, we benchmark against RootClam~\cite{han2023root} and CausalRCA~\cite{budhathoki2022causal}, which are state-of-the-art models in RCA. 

\subsection{Results and Analysis}


\subsubsection{Anomaly Detection}

We train our model's anomaly detection performance using a dataset of $10000$ data points collected during normal operating conditions. To test the model, we create a mixed dataset containing samples from both normal and anomaly periods. 
In this section, we do not distinguish between anomaly types; all anomaly samples (both cyber and measurement anomaly) are marked as $1$, while samples from non-anomaly time periods are marked as $0$.
Anomaly scores are calculated for all samples using Eq. (\ref{eq:anomaly-score}), with samples with anomaly scores exceeding a specified threshold classified as anomalies. 
We conduct experiments with $\alpha$ values of $0.5$ and $1.0$ in Eq. (\ref{eq:anomaly-value}) to assess the model's sensitivity under different level of anomalies. 
Performance is measured using precision, recall, and F1 scores across three distinct datasets, with results shown in Table \ref{table:main-result}.

Table~\ref{table:main-result} shows that our method ICODE consistently outperforms the classical Deep SVDD and anomaly transformer model across all metrics, likely due to our ODE-based architecture simplifying the learning process. 
Compared to the Anomaly Transformer, our model shows slightly better performance overall, though marginally lower precision and F1 scores are observed with $\alpha=0.5$. 
This suggests that while our approach is generally more effective, it may be less adept at capturing subtle anomalies that occur with lower $\alpha$ values, compared to deep learning models. 
The superior performance of our model can be attributed to its ODE-based design, which leverages the quadratic nature of ODEs to more accurately model system dynamics and detect anomalies.

\subsubsection{Root Cause Localization}

\vspace{0.27in}
\begin{table*}[!h]
\scriptsize
\centering
\caption{Performance of ICODE compared to baseline models in root cause localization.}
\begin{tabular}{llcccccccccc}
    \midrule
    & & & \multicolumn{3}{c}{\textbf{Lotka-valterra}} & \multicolumn{3}{c}{\textbf{Lorenz-96}} & \multicolumn{3}{c}{\textbf{Reaction-diffusion}} \\
    \cmidrule(r){4-12}
    & & & Top 1 & Top 3 & Top 5 & Top 1 & Top 3 & Top 5 & Top 1 & Top 3 & Top 5  \\
    \midrule
    \multirow{3}{*}{$\alpha = 0.5$} & CausalRCA & & 0.1565 & 0.3424 & 0.4805 & 0.6766 & 0.6928 & 0.7369 & 0.9235 & 0.9481 & 0.9655 \\
    & RootClam & & 0.3360 & 0.4247 & 0.5699 & 0.6260 & 0.6315 & 0.6724 & 0.9395 & 0.9681 & 0.9709  \\    
    & ICODE & &  \textbf{0.6713} & \textbf{0.7124} & \textbf{0.7246}  & \textbf{0.7546} & \textbf{0.7742} & \textbf{0.7883} & \textbf{0.9645} & \textbf{0.9783} & \textbf{0.9748} \\
    \midrule
    \multirow{3}{*}{$\alpha = 1$} & CausalRCA & & 0.2724 & 0.3922 & 0.5269 & 0.6978 & 0.7485 & 0.7538 & 0.8857 & 0.9563 & 0.9736 \\
    & RootClam & & 0.5546 & 0.5630 & 0.6554 & 0.6519 & 0.7586 & 0.7936 & 0.6554 & 0.9763 & 0.9788 \\
    & ICODE & & \textbf{0.7710} & \textbf{0.8109} & \textbf{0.8385} & \textbf{0.7632} & \textbf{0.7890} & \textbf{0.8241} & \textbf{0.9798} & \textbf{0.9911} & \textbf{0.9952} \\
    \midrule
\end{tabular}
\label{table:rca-result}
\end{table*}

We evaluate ICODE's root cause localization performance using simulated anomaly data. 
Based on the model trained in the previous section, we retrain it on $500$ samples from each anomaly period to obtain a new causal graph. 
In Eq. (\ref{eq:anomaly-type}), $m=10$ and $\Bar{\gamma}$ is the threshold of largest $10$ elements.
An anomaly is categorized as a measurement anomaly if $M(C, C') \geq 0.8$; otherwise, it is classified as a cyber anomaly. 
We then apply Eq. (\ref{eq:get-root-cause}) or Eq. (\ref{eq:get-root-cause-2}) to localize the root cause, ranking variables based on their root cause scores. 
Performance is assessed using Top-k accuracy, which measures whether the true root cause is included in the output top k set. 
This process is repeated for each anomaly period, with results summarized in Table~\ref{table:rca-result}.

Table~\ref{table:rca-result} shows that ICODE outperforms baseline models CausalRCA and RootClam, achieving higher Top-1, Top-3, and Top-5 accuracy across all three simulated datasets. 
The reaction-diffusion system, with its linear ODE function, yields relatively good results for all models. 
Performance degrades for Lotka-Volterra and Lorenz-96 systems, which we attribute to its more complex dynamics making causality inference more challenging.

\begin{figure}[!ht]
\label{fig:demo-class}
    \centering
    \subfloat[]
    {\label{fig:gt-g}
    \raisebox{0.9cm}{\includegraphics[width=0.4\columnwidth]{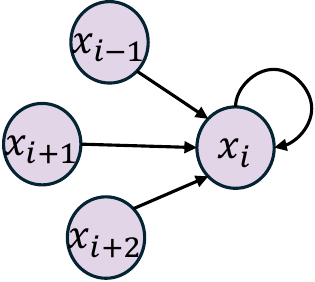}}}
    \hfill 
    \subfloat[]
    {\label{fig:gt-a}
    \hspace{-.11in}\includegraphics[width=0.55\columnwidth]{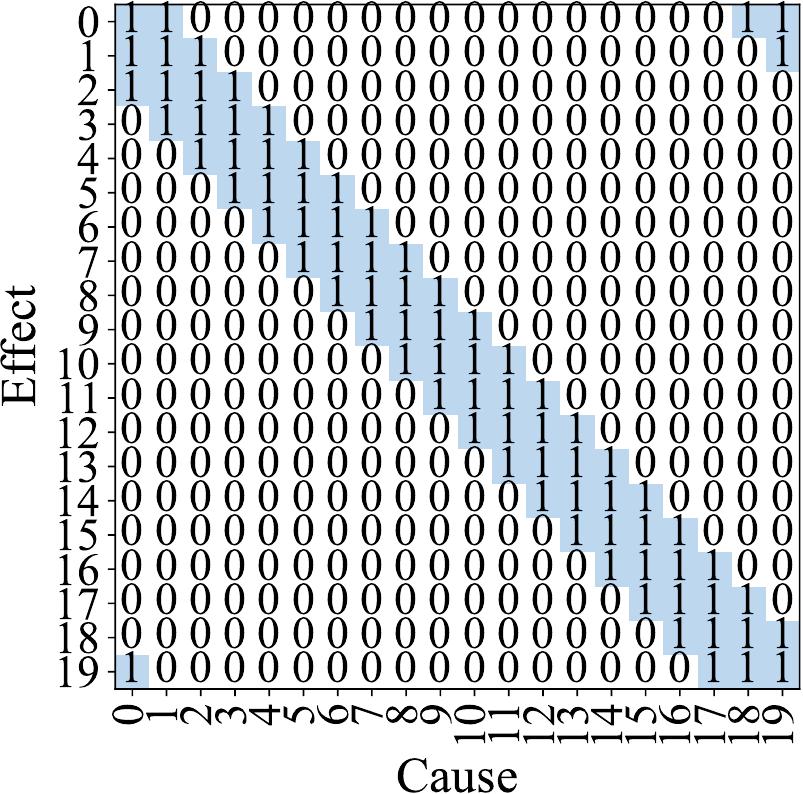}}
    \caption{(a) Graph representation $\mathcal{G}$ of causality relationship in the Lorenz-96 system.
    (b) Adjacency matrix of the causality graph $\mathcal{G}$. 
    }
    \label{fig:gt}
\end{figure}

These results support our hypothesis that learning the ODE in dynamical systems can improve anomaly detection and root cause localization. 
ICODE's approach aligns with the mechanism of anomaly occurrence: exceptions due to ODE changes are located by learning about the ODE and inferring the root cause through changes in ICODE parameters. 
In contrast, baseline models lack this grounding in dynamical systems science. CausalRCA's causal inference model may struggle to learn causal relationships in this context, while RootClam's use of distribution changes for root cause localization may be less effective in dynamical systems. 
In the following part, we will discuss the anomaly type classification based on the proposed theorems.

\subsubsection{Anomaly Type Classification}
To better demonstrate the impact of different anomalies, we provide an intuitive description of anomaly behavior using the Lorenz-96 system with $\alpha=5$, demonstrating strong cyber and measurement anomalies in a dynamical system. 
Figure \ref{fig:gt-g} illustrates the causality relationship in the Lorenz-96 system, with the corresponding adjacency matrix shown in Figure \ref{fig:gt-a}. 

To analyze anomaly patterns, we simulate $500$ data points each for cyber and measurement anomalies with the root cause at variable 10. 
For clarity, we employ K-means clustering to categorize weights into two clusters: high dependency ($1$) and low dependency ($0$). 
Figures \ref{fig:result-lorenze-a} and \ref{fig:result-lorenze-b} depict the resulting adjacency matrices during measurement anomalies and cyber anomalies, respectively.
Our analysis reveals distinct patterns for each anomaly type. 
In measurement anomalies (Figure \ref{fig:result-lorenze-a}), changes in causality relationships are localized, with significant alterations primarily in variable $10$'s dependencies compared to the normal state (Figure \ref{fig:gt-a}). 
Conversely, cyber anomalies (Figure \ref{fig:result-lorenze-b}) exhibit a more diffuse pattern, where variables proximal to the root cause show more substantial changes.

These observations form the basis for our anomaly type classification strategy. It is important to note that this example represents an ideal case with optimal training conditions. 
In real-world cases, particularly with lower $\alpha$ values (e.g., $\alpha=0.5$), the distinctions may be less pronounced. 
Since cyber and measurement anomalies are specifically defined in our work, there are no existing approaches available for comparison.

\begin{table}[h!]
\centering
\caption{Performance of ICODE in anomaly classification.}
\label{table:rca-class}
\begin{tabular}{lccc}
    \toprule
    $\alpha$ & Lotka-Volterra & Lorenz-96 & Reaction-Diffusion \\
    \midrule
    0.5 & 0.929 & 0.952 & 0.971 \\
    1   & 0.916 & 0.961 & 0.974 \\
    \bottomrule
\end{tabular}
\end{table}

Table~\ref{table:rca-class} presents the classification results based on the criteria mentioned in root cause localization.
The results in Table~\ref{table:rca-class} demonstrate ICODE's high accuracy in anomaly type classification, providing reliable insights that are not available from baseline models.
These findings further support our hypothesis that changes in causality relationships can effectively address three key tasks: anomaly detection, root cause localization, and anomaly classification. 
We believe that such explanatory capabilities are important to domain experts in system diagnostics, offering a level of interpretability crucial for real-world applications in complex dynamical systems.

\begin{figure}[!ht]
\label{fig:result-lorenze}
    \centering
    \subfloat[Measurement Anomaly]
    {\label{fig:result-lorenze-a}
    \includegraphics[width=0.485\columnwidth]{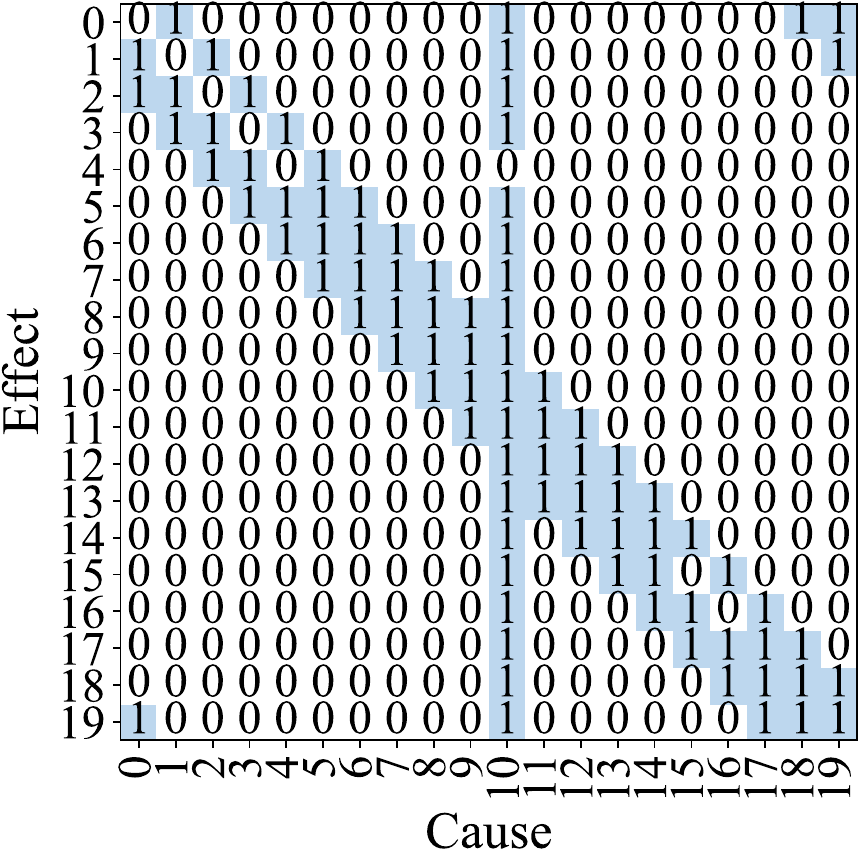}}
    \hfill%
    \subfloat[Cyber Anomaly]
    {\label{fig:result-lorenze-b}
    \includegraphics[width=0.485\columnwidth]{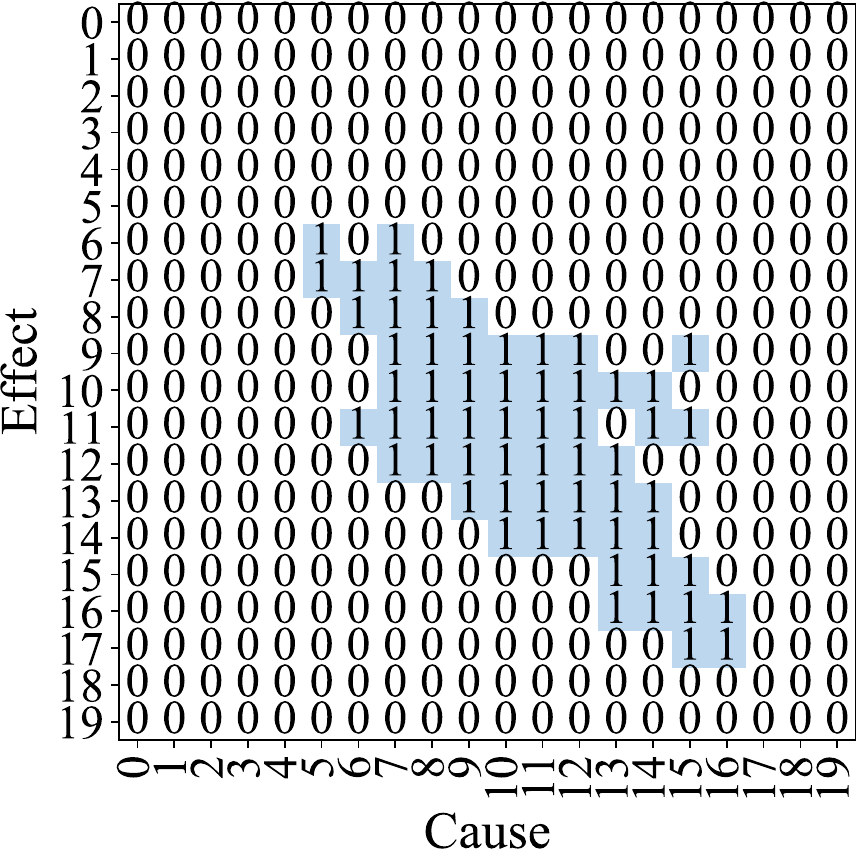}}
    \caption{Causality relationships during measurement and cyber anomalies when variable $10$ is the root cause. (a) The causality changes caused by the measurement anomaly are concentrated in the root cause variable. (b) The cyber anomaly changes the causal graph globally. 
    }
\end{figure}

\section{Conclusion}

In this paper, we present ICODE, a novel approach integrating anomaly detection, classification and root cause localization in dynamical systems. Leveraging Neural ODEs to model system dynamics and interpret causality relationships, ICODE demonstrates superior performance over state-of-the-art methods. Its key contribution lies in its interpretability, providing valuable insights into system behavior during anomalous events by analyzing changes in learned causality relationships. While showing promising results, future work should focus on testing ICODE on more complex real-world datasets and exploring its applicability across different domains. Overall, ICODE represents a significant advancement in anomaly analysis for dynamical systems, offering a comprehensive, interpretable solution that bridges the gap between anomaly detection with root cause analysis, and explainable AI.

\section{Acknowledgements}
This work was sponsored by a Research Futures grant from Lehigh University and the Office of Naval Research (ONR) under grant number N00014-22-1-2626.

\bibliography{aaai25}

\section{Proof}

\setcounter{theorem}{0}
\begin{theorem}
\label{theorem:measurement-appendix}
Consider a dynamical system $X(t)$ as learned using Eq. (\ref{eq:ICODE}), where $\mathcal{G}=(X(t), \mathcal{E})$ denotes the underlying dependency graph. 
Let $C(\cdot,\cdot)$ and $C'(\cdot,\cdot)$ denote the causality matrices extracted using non-anomalous and anomalous data, respectively. 
Assuming $C(i,j)\neq 0$ iff $(i,j)\in \mathcal{E}$, 
and {the diagonal elements of $C$ are much larger than the off-diagonal elements,} $C(i,i) \gg C(i,j), i \neq j$, then
(1) If the anomaly is a measurement anomaly at variable $k$, then $|C(k,j)-C’(k,j)| \gg |C(k',j)-C’(k',j)| > 0$, where $(k, k') \in \mathcal{E}$.
(2) If the anomaly is a cyber anomaly at variable $k$, then $|C(k,j)-C’(k,j)| > 0$ only if $(k,j)\in \mathcal{E}$.
\end{theorem}

\begin{proof}
Part 1 Measurement Anomaly: 

We start by assuming the model defined in Eq. (\ref{eq:ICODE}) learns the dynamical system with high accuracy, and $\Phi(\cdot)$ represents the learned ODE governing the dynamical system:
\begin{equation}
\label{eq:ground-truth-ODE}
\begin{aligned}
    \frac{dX(t)}{dt} & = \Phi_{\theta}(X(t))X(t), \\
    X(t+1) & = X(t) + \int_{t}^{t+1}  \Phi_{\theta}(X(\tau))X(\tau) d \tau, 
\end{aligned}
\end{equation}
where $X(t)$ is the state of the system at time $t$, and $\Phi_{\theta}(\cdot): \mathbb{R}^{p} \rightarrow \mathbb{R}^{p \times p}$ is a neural network learning the dependency with parameter $\theta$.


We denote the neural network trained on the system with anomalies as $\Phi_{\theta'}(\cdot)$. When the system undergoes a measurement anomaly, the ODE in Eq. (\ref{eq:ground-truth-ODE}) modifies to:
\begin{equation}
    \Tilde{X}(t+1) = \Tilde{X}(t) + \left( \int_{t}^{t+1} \Phi_{\theta'}(\Tilde{X}(\tau))(\Tilde{X}(\tau)) d \tau \right), 
\end{equation}
where $\Tilde{X}$ represents the reading with anomaly, and $Z(t)=A(X(t))-X(t)$ is the impact of the anomaly on the variable observation, with only one non-zero element.

By replacing $\Tilde{X}(t) = X(t) + Z$, we have
\begin{equation}
\label{eq:thm-1-proof}
    {X}(t+1) - {X}(t) = \left( \int_{t}^{t+1} \Phi_{\theta'}(\Tilde{X}(\tau))({X}(\tau) + Z) d \tau \right), 
\end{equation}
where ${X}(t+1) - {X}(t)$ could be calculated using the evolution in the normal period:
\begin{equation}
\label{eq:thm-1-proof2}
    {X}(t+1) - {X}(t) = \left( \int_{t}^{t+1} \Phi_{\theta}({X}(\tau)){X}(\tau) d \tau \right).
\end{equation}

By Eq. (\ref{eq:thm-1-proof}) and Eq. (\ref{eq:thm-1-proof2}), and since the equation holds for any $t$:
\begin{equation}
    \Phi_{\theta}({X}(t)){X}(t) = \Phi_{\theta'}(\Tilde{X}(t))({X}(t) + Z).
\end{equation}

The difference between $\Phi_{\theta}(\cdot)$ and $\Phi_{\theta'}(\cdot)$ can be expressed as:
\begin{equation}
\label{eq:theorem-1}
    (\Phi_{\theta}(X(t)) - \Phi_{\theta'}(\Tilde{X}(t))) X(t) = \Phi_{\theta'}(X(t)) Z.
\end{equation}

Since $Z$ contains only one non-zero element (denoted $z$), the product $\Phi_{\theta'}(X(t))Z$ is sparse, with only a few non-zero components corresponding to the root cause $k$ and its neighbor $k'$ where $(k, k') \in \mathcal{E}$.

Given that the diagonal entries of $\Phi_{\theta'}(X(t))$ are much larger than the off-diagonal ones, i.e., $C(i,i) \gg C(i,j)$, we have:
\begin{equation}
\label{eq:eq23}
\begin{aligned}
    (\Phi_{\theta'}(X(t))Z)_k & = \Phi_{\theta'}(X(t))_{k,k} z \gg \\ \Phi_{\theta'}(X(t))_{k,k'} z & = (\Phi_{\theta'}(X(t))Z)_{k'}.    
\end{aligned}
\end{equation}

This implies that the $k$-th row of the matrix difference $|\Phi_{\theta}(X(t)) - \Phi_{\theta'}(X(t))|$ will exhibit a significantly larger magnitude than the others. Specifically,
\[
|C(k,j) - C'(k,j)| \gg |C(k',j) - C'(k',j)| > 0,
\]
where $k$ is the root cause and $(k,j) \in \mathcal{E}$.

Part 2 Cyber Anomaly: 

We consider the trained model representing a cyber anomaly, defined similarly to the measurement anomaly case with $Z(t) = A(X(t))-X(t)$:
\begin{equation}
X(t+1) = X(t) +  \int_{t}^{t+1} \Phi_{\theta}(X(\tau)) \left( X(\tau)+Z(\tau)\right) d \tau ,
\end{equation}
where $Z$ represents the anomaly, with only one non-zero element.
The difference between the neural networks $\Phi_{\theta}(X(t))$ and $\Phi_{\theta'}(X(t))$ is given by:
\begin{equation}
\label{eq:theorem2-2}
(\Phi_{\theta}(X(t)) - \Phi_{\theta'}(X(t))) X(t) = \int_{t}^{t+1} \Phi_{\theta}(X(\tau)) Z(\tau) d\tau.
\end{equation}
To prove our theorem, we set up a proof by contradiction:
Suppose there exists $i=k$ and $(i,j) \in \mathcal{E}$ such that $C(i,j) - C'(i,j) = 0$. This would imply that $(\Phi_{\theta}(X(t)) - \Phi_{\theta'}(X(t))) = 0$, leading to:
\begin{equation}
\label{eq:contradiction}
\int_{t}^{t+1} \Phi_{\theta}(X(\tau)) Z(\tau) d\tau = 0.
\end{equation}
However, Eq. (\ref{eq:contradiction}) implies that $Z(t)=0$, which contradicts the definition of an anomaly where $A(x) \neq x$. Therefore, our initial assumption must be false, proving the theorem.




\end{proof}

\end{document}